\documentclass[10pt,journal,compsoc]{IEEEtran}
\ifCLASSOPTIONcompsoc
  \usepackage[nocompress]{cite}
\else
  \usepackage{cite}
\fi
\ifCLASSINFOpdf
\else
\fi
\usepackage[abs]{overpic} 
\usepackage{amsfonts}
\usepackage{amsmath}
\usepackage{amsthm}
\usepackage{txfonts}
\usepackage{enumerate}
\usepackage{mathrsfs}
\usepackage{pifont} \usepackage[all]{xy}
\usepackage{graphicx}
\usepackage{color}
\usepackage{ulem}
\usepackage{comment}
\usepackage{here}
\usepackage{tikz}
\usepackage{hyperref}
\usepackage{subfigure}
\usepackage[whole]{bxcjkjatype}  
\usepackage{algorithmic}
\usepackage{algorithm}
\usepackage{caption}

\DeclareMathAlphabet{\mathfrak}{U}{euf}{m}{n}
\SetMathAlphabet{\mathfrak}{bold}{U}{euf}{b}{n}

\theoremstyle{definition}

\renewcommand{\phi}{\varphi}
\renewcommand{\epsilon}{\varepsilon}

\newcommand{\blue}{\textcolor{blue}}

\newtheorem{df}{Definition}[section]

\newtheorem{prp}[df]{Proposition}

\newtheorem{ex}[df]{Example}
\newtheorem{rem}[df]{Remark}

\makeatletter
\renewcommand{\theequation}{%
	\thesection.\arabic{equation}}
\@addtoreset{equation}{section}
\makeatother

\begin{document}
%
\title{Image recognition via Vietoris-Rips complex}
%
%
%
%

\author{Yasuhiko~Asao,~\IEEEmembership{Fukuoka~University,}
        Jumpei~Nagase,~\IEEEmembership{Shibaura~Institute~of~Technology
        Ryotaro~Sakamoto,~\IEEEmembership{RIKEN Center for Advanced Intelligence Project}
        and
        Shiro~Takagi,~\IEEEmembership{Independent Researcher}
        }
\IEEEcompsocitemizethanks{\IEEEcompsocthanksitem Y. Asao is with Fukuoka university, 8-19-1, Nanakuma, Fukuoka, Japan.\protect\\
E-mail: asao@fukuoka-u.ac.jp
\IEEEcompsocthanksitem J. Nagase is with Shibaura Institute of Technology, 307 Fukasaku, Minuma-ku, Saitama City, Saitama, 337-8570, Japan.\protect\\
E-mail:nb20106@shibaura-it.ac.jp
\IEEEcompsocthanksitem R. Sakamoto is with RIKEN, 2-1 Hirosawa, Wako, Saitama, 351-0198, Japan.\protect\\
E-mail: r-sakamoto@keio.jp
\IEEEcompsocthanksitem S. Takagi is an independent researcher.\protect\\
E-mail: takagi4646@gmail.com

\IEEEcompsocthanksitem  Authors are listed in alphabetical order.
}
}

\IEEEtitleabstractindextext{%
\begin{abstract}
Extracting informative features from images has been of capital importance in computer vision. In this paper, we propose a way to extract such features from images by a method based on algebraic topology. 
To that end, we construct a weighted graph from an image, which extracts local information of an image. 
By considering this weighted graph as a pseudo-metric space, we construct a Vietoris-Rips complex with a parameter $\varepsilon$ by a well-known process of algebraic topology. 
We can extract information of complexity of the image and can detect a sub-image with a relatively high concentration of information from this Vietoris-Rips complex. 
The parameter $\varepsilon$ of the Vietoris-Rips complex produces robustness to noise. 
We empirically show that the extracted feature captures well images' characteristics.
\end{abstract}

}

\maketitle

\IEEEdisplaynontitleabstractindextext

%
\IEEEpeerreviewmaketitle

\IEEEraisesectionheading{
\section{Introduction}
\label{section:introduction}
}
Image processing is one of the major areas of computer science. 
This area has been growing rapidly in recent years with the development of machine learning researches. 
It encompasses image recognition \cite{Khan20}, object detection \cite{Zou19}, and segmentation \cite{Minaee21} and they are also being actively pursued and utilized throughout the industry \cite{Litjens17,Cheng20}. One of the important subjects is to extract a geometric property and a feature invariant under some operations from data \cite{Bronstein21}. 
This is also of great importance to the machine learning research community. 
While many pieces of researches in this direction have a tremendous impact on practical applications, they are usually based on heuristic engineering methods. On the other hand, the method proposed in this paper can reduce such heuristics.

Recently, topological data analysis, application of algebraic topology to data analysis, is getting a lot of attention \cite{Edelsbrunner00}. 
It is an effective mathematical technique to extract features from data.
In particular, persistent homology is intensively studied and employed for various applications \cite{Edelsbrunner08}. 
Since topological data analysis is based on a mathematically rigorous procedure, it provides theoretically interpretable features in a less heuristic manner. Our method can be classified into topological data analysis and is enjoying such benefits.

In the present paper, we propose a way to detect areas where many colors concentrate in an image by using powerful tools in algebraic topology. 
To be precise, we translate our visual recognition process as {\it simplicial complex}, and we consider 0-th persistent homology of it for abstraction.
This translation is based on the observation that the human's process of recognizing an image from global to local parts corresponds to constructing a structure of simplicial complex on a set of areas in the image. 

For the construction of the simplicial complex from an image, we construct a weighted graph, whose weights are characterized by the color information of pixels of the image. This graph is similar to, or in a sense a generalization of, the \textit{quadtree}\cite{Finkel74}. 
Then we regard this weighted graph as a pseudo-metric space by considering the shortest path length of the weighted graph as a metric, and we obtain the desired object by a certain procedure.
The simplicial complex which we construct here is called \textit{Vietoris-Rips complex}, a basic tool in algebraic topology \cite{Carlsson09}. 
The Vietoris-Rips complex constructed from an image retains important characteristics of the original image and represents the image in a mathematically and computationally tractable manner. 
For example, we can naturally recover the \textit{complexity} of images defined in previous literature \cite{Asao20} by considering its persistent homology (Proposition \ref{prop:VR-depth}). 
Furthermore, we can find areas where information highly concentrate in the image by investigating maximal simplices of the Vietoris-Rips complex. 
We demonstrate that this procedure captures well images' characteristics by applying it to object detection.

Recently, many researchers study the explication of explanatory variables of prediction via machine learning \cite{Das20}. 
For example, some researchers propose a way to visualize features of data that greatly influence the model's prediction \cite{Zeiler14}.
However, features extracted from data are currently not given a theoretical interpretation. 
On the other hand, our method is one of the non-learning methods, which presents a universal and new perspective in image analysis, and has a theoretical interpretation. 
In addition, from the way the weighted graph is constructed, our method is invariant to image rotation. 
Moreover, the method for constructing weighted graphs from images is itself a novel idea. 

The structure of the present paper is as follows. 
In \S \ref{section:introduction} we describe the background and an overview of the current paper. 
We  review related literature in \S \ref{section:related-works}. 
In \S \ref{section:construction}, we will explain how to construct the weighted graph, the pseudo-metric space, and the Vietoris-Rips complex from an image. 
In \S \ref{section:connected-component}, we introduce how to extract the feature from an image by using the associated Vietoris-Rips complex. 
In \S \ref{section:experiment}, we demonstrate the result of our numerical experiment, where we show that our method can be used for object detection from photographic images. 

\section{Related Works}
\label{section:related-works}

Image processing and analysis by representing an image as a graph has been studied intensively \cite{Lezoray17}. 
For instance, an image can be represented as a 2D lattice or a skeleton graph.
While most of the constructions of graphs that we know use only pixels, our graphs have vertices corresponding to sub-regions of a given image. Also, we can completely automatically construct our graphs, while the previous ones are constructed somewhat heuristically. 

Barroso et al. utilize a graph called {\it quadtree} for object detection \cite{Barroso04}. Region quadtree is another graph construction where vertices correspond to sub-regions \cite{Samet84}. 
Our construction of graphs differs from this one at least in terms of the following point. That is, all square sub-regions correspond to vertices in our construction, while only sub-regions which are cut into quarters correspond to vertices in quadtrees. As a result, the obtained graphs of us are not tree. As we explain in Remark \ref{generalization}, our graph contains quadtree as a subgraph. This property gives the graph a richer geometric structure, and it may extract more features of images than quadtrees do.



Several studies apply persistent homology for image analysis \cite{Carlsson08,Romero14,Giansiracusa17,Garin19,Chung18,Kovalevsky89,Luo19,Letscher07}. Namely, they extract features of images by analyzing geometric structures constructed from images. For example, Yang--Wohlberg \cite{Yang19} and Garin--Tauzin \cite{Garin19} utilize the extracted features for MNIST \cite{Lecun98} classification, Chung et al. \cite{Chung18} use them for skin disease image segmentation and classification. Our research is in the same direction. Namely, we construct a weighted graph characterized by the color differences between sub-regions. Then, we consider Vietoris-Rips complex based on this graph and analyze its geometric structure by calculating its 0-th persistent homology. We have no limitation nor need preprocessing on images, while some of the above previous works need.

Recently, learning-based methods such as deep learning have been developed in the field of image recognition \cite{Krizhevsky12,He16}. 
Although those methods have shown high performance by automatically learning features from data, there are some problems with the obtained features. 
For example, convolutional neural networks are known to be not invariant with respect to some geometric transformations \cite{Sabour17} and the learned features are not necessarily interpretable \cite{Zhang20}.
Therefore, acquiring interpretable features which are invariant with respect to some geometric transformations is attracting attention as a complementary way to learning-based methods \cite{Robert14,Kuzminykh18,Kosiorek19,Shorten19,Bronstein21}.
As we mentioned in the introduction, we propose a method to extract the feature which is mathematically interpretable and invariant to image rotations.

\section{Construction of the weighted graph and the Vietoris-Rips complex}
\label{section:construction}
As mentioned in the introduction, our approach is based on the observation that the process of recognizing an image from global to local parts corresponds to constructing a structure of simplicial complex on a set of areas in the image. 
For the accurate execution of our idea, we should work in a mathematically rigorous way.
To that end, in this section, we state our formulation of images, and we recall the definitions of a simplicial complex and a Vietoris-Rips complex. 
After that, we explain our construction of the weighted graph and the Vietoris-Rips complex associated with an image.


\subsection{Construction of the weighted graph and the pseudo-metric space}

For any positive integer $M$, we define 
\[
[M] := \{0, 1, \ldots, M-1\}.
\]

\begin{df}[Image]
Throughout this paper, for any positive integers  $M$, $N$, and $C$, we consider a map 
\[
f \colon [M]\times [N] \longrightarrow [C]
\]
as a $C$-color $M\times N$-pixel image. 
Here we regard $[C]$ as a color set (representing a $C$-step gray scale). 
In the following, we call such a map just an {\it image}. We also call the cardinality of the set 
$\# f([M]\times [N])$ the {\it number of colors} of the image $f$.
\end{df}

\begin{ex}\label{eg1}
The following $3\times 3$-pixel black-and-white image is identified with the map represented on the right by associating $0$ and $1$ to black and white, respectively.
\begin{figure}[H]
\centering
\begin{tikzpicture}
\draw[thick] (0,0) rectangle (3,3) ;
\filldraw[thick] (0, 2) rectangle (3,3);
\filldraw[thick] (1, 1) rectangle (3,2);
\filldraw[thick] (2, 0) rectangle (3,1);
\end{tikzpicture}
$\  $
\begin{tikzpicture}
\draw[thick] (0,0) rectangle (3,3) ;
\coordinate (a) at (0, 1.5) node [left] at (a) {$=\ \ \ $};
\coordinate (a) at (0.5, 0.5) node at (a) {{\Large 1}};
\coordinate (a) at (0.5, 1.5) node  at (a) {{\Large 1}};
\coordinate (a) at (1.5, 0.5) node at (a) {{\Large 1}};
\coordinate (a) at (1.5, 1.5) node  at (a) {{\Large 0}};
\coordinate (a) at (1.5, 2.5) node  at (a) {\Large 0};
\coordinate (a) at (0.5, 2.5) node  at (a) {\Large 0};
\coordinate (a) at (2.5, 2.5) node  at (a) {\Large 0};
\coordinate (a) at (2.5, 1.5) node  at (a) {\Large 0};
\coordinate (a) at (2.5, 0.5) node  at (a) {\Large 0};

\draw[thick](0, 1) -- (3, 1);
\draw[thick](0, 2) -- (3, 2);
\draw[thick](1, 0) -- (1, 3);
\draw[thick](2, 0) -- (2, 3);
\end{tikzpicture}
\caption{}
\end{figure}
\noindent
That is to say, we identify the $3\times 3$-pixel black-and-white image on the left with the map 
\[
f \colon [3]\times [3] \longrightarrow [2]; \,  f(i, j) := \begin{cases} 1 & (i, j) \in \{ (0, 1), (0, 2), (1, 2)\}, \\ 0  & {\rm others.}\end{cases}
\]
\end{ex}
\begin{df}[Square]
We define 
\[
[M, a, m] := \{m, m+1, \dots, m+a-1\} \subset [M]
\]
for any integers $a$ and $m$ satisfying $1 \leq a \leq M$ and $0 \leq m < M - a$. 
We also define 
\begin{align*}
    \square_{M, N} 
    &:= \left\{ [M, a, m] \times [N, a, n]  \,\, \middle|  
    \begin{array}{l} 
    1\leq a \leq \min\{M, N\}, 
    \\
    0 \leq m \leq M-a, 
    \\
    0 \leq n \leq N-a
    \end{array}
    \right\}. 
\end{align*}
We call each element of $\square_{M, N}$ a {\it square}. 
We define the {\it size} of a square $[M, a, m] \times [N, a, n]$ to be $a$.
We say that a set $[M, a, m] \times [N, b, n]$ is a {\it rectangle} if it is not necessarily a square (cf. Figure\ref{fig}). 
\end{df}
\begin{figure}[H]
\centering
\begin{tikzpicture}
\draw[very thick] (0,0) rectangle (3,3) ;
\draw[gray][ultra thin] (0.75,0) -- (0.75,3);
\draw[gray][ultra thin] (1.5,0) -- (1.5,3);
\draw[gray][ultra thin] (2.25,0) -- (2.25,3);
\draw[gray][ultra thin](0, 0.75) -- (3, 0.75);
\draw[gray][ultra thin] (0, 1.5) -- (3, 1.5);
\draw[gray][ultra thin] (0, 2.25) -- (3, 2.25);
\draw[very thick][dotted][red] (0.75, 0.75) rectangle (3, 2.25);
\end{tikzpicture}
\caption{
The square surrounded by the solid  thick line
represents 
$[4]\times [4]$, 
and the rectangle surrounded the dotted line represents 
$[4, 3, 1]\times[4, 2, 1]$. 
}
\label{fig}
\end{figure}
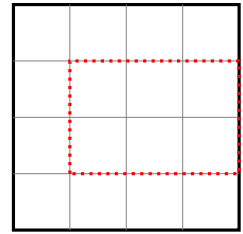

Next, we construct a weighted graph $(V_{M, N}, E_{M, N}, w_f)$ from a $C$-color $M \times N$-pixel image $f \colon [M] \times [N] \longrightarrow [C]$. 

\begin{df}[Weighted graph]\label{wgraph}
We define the set of vertices $V_{M,N}$ by 
\[
V_{M,N} := \square_{M, N}, 
\] 
and we define the set of edges $E_{M, N}$ by
\begin{align*}
    E_{M, N} := 
\left\{ \{A, B\} \,\, \middle| 
\begin{array}l
A \subset B \in \square_{M, N},  
\\
\textrm{the size of $B$}  = 1 + \textrm{the size of $A$} 
\end{array}
\right\}. 
\end{align*}
We define the weight $w_f \colon E_{M, N} \longrightarrow \mathbb{Z}_{\geq 0}$ associated with the image $f$ by 
\[
w_f(\{A, B\}) = |\# f(B) - \# f(A)| \geq 0.
\]
\end{df}

\begin{ex}\label{exgraph}
We illustrate the graph $(V_{3,3}, E_{3,3})$.  
The top vertex corresponds to $[3] \times [3]$, each of the vertices in the middle row corresponds to each of the squares in $\Box_{3,3}$ with size $2$, and each of the bottom vertices corresponds to each of the pixels of $[3] \times [3]$. 
\begin{figure}[H]
\centering
\begin{tikzpicture}
\draw[thick] (0,0) rectangle (3,3) ;
\draw[->][thick] (1.5,-0.75)--(1.5,-1.875);
\node(P) at (1.5,-2.25) {};
\draw[thick](0, 1) -- (3, 1);
\draw[thick](0, 2) -- (3, 2);
\draw[thick](2, 0) -- (2, 3);
\draw[thick](1, 0) -- (1, 3);
\end{tikzpicture}
\begin{tikzpicture}
\coordinate (a) at (0, 1.5)  ;
\coordinate (b) at (-2, 0)  ;
\coordinate (c) at (-1, 0) ;
\coordinate (d) at (1, 0) ;
\coordinate (e) at (2, 0) ;
\coordinate (f) at (-4, -1.5)  ;
\coordinate (g) at (-3, -1.5) ;
\coordinate (h) at (-2, -1.5) ;
\coordinate (i) at (-1, -1.5) ;
\coordinate (j) at (0, -1.5) ;
\coordinate (k) at (1, -1.5) ;
\coordinate (l) at (2, -1.5) ;
\coordinate (m) at (3, -1.5) ;
\coordinate (n) at (4, -1.5) ;
\fill (a) circle (2pt);
\fill (b) circle (2pt);
\fill (c) circle (2pt);
\fill (d) circle (2pt);
\fill (e) circle (2pt);
\fill (f) circle (2pt);
\fill (g) circle (2pt);
\fill (h) circle (2pt);
\fill (i) circle (2pt);
\fill (j) circle (2pt);
\fill (k) circle (2pt);
\fill (l) circle (2pt);
\fill (m) circle (2pt);
\fill (n) circle (2pt);
\draw[thick](a) -- (b);
\draw[thick](a) -- (c);
\draw[thick](a) -- (d);
\draw[thick](a) -- (e);
\draw[thick](b) -- (f);
\draw[thick](b) -- (g);
\draw[thick](b) -- (j);
\draw[thick](b) -- (i);
\draw[thick](c) -- (h);
\draw[thick](c) -- (g);
\draw[thick](c) -- (j);
\draw[thick](c) -- (k);
\draw[thick](d) -- (n);
\draw[thick](d) -- (m);
\draw[thick](d) -- (j);
\draw[thick](d) -- (k);
\draw[thick](e) -- (l);
\draw[thick](e) -- (m);
\draw[thick](e) -- (j);
\draw[thick](e) -- (i);
\end{tikzpicture}
\caption{}
\end{figure}
The following weighted graph corresponds to $(V_{3,3}, E_{3,3}, w_f)$ obtained from the image $f \colon [3] \times [3] \longrightarrow [2]$ defined in Example~\ref{eg1}.
\begin{figure}[H]
\centering
\begin{tikzpicture}
\draw[thick] (0,0) rectangle (3,3) ;
\draw[->][thick] (1.5,-0.75)--(1.5,-1.875);
\node(P) at (1.5,-2.25) {};

\draw[thick] (0,0) rectangle (3,3) ;
\coordinate (a) at (0.5, 0.5) node at (a) {{\Large 1}};
\coordinate (a) at (0.5, 1.5) node  at (a) {{\Large 1}};
\coordinate (a) at (1.5, 0.5) node at (a) {{\Large 1}};
\coordinate (a) at (1.5, 1.5) node  at (a) {{\Large 0}};
\coordinate (a) at (1.5, 2.5) node  at (a) {\Large 0};
\coordinate (a) at (0.5, 2.5) node  at (a) {\Large 0};
\coordinate (a) at (2.5, 2.5) node  at (a) {\Large 0};
\coordinate (a) at (2.5, 1.5) node  at (a) {\Large 0};
\coordinate (a) at (2.5, 0.5) node  at (a) {\Large 0};

\draw[thick](0, 1) -- (3, 1);
\draw[thick](0, 2) -- (3, 2);
\draw[thick](1, 0) -- (1, 3);
\draw[thick](2, 0) -- (2, 3);
\end{tikzpicture}


\begin{tikzpicture}
\coordinate (a) at (0, 1.5)  ;
\coordinate (b) at (-2, 0)  ;
\coordinate (c) at (-1, 0) ;
\coordinate (d) at (1, 0) ;
\coordinate (e) at (2, 0) ;
\coordinate (f) at (-4, -1.5)  ;
\coordinate (g) at (-3, -1.5) ;
\coordinate (h) at (-2, -1.5) ;
\coordinate (i) at (-1, -1.5) ;
\coordinate (j) at (0, -1.5) ;
\coordinate (k) at (1, -1.5) ;
\coordinate (l) at (2, -1.5) ;
\coordinate (m) at (3, -1.5) ;
\coordinate (n) at (4, -1.5) ;
\coordinate (p) at (-1, 0.75) node [left] at (p) {$0$};
\coordinate (p) at (-0.5, 0.75) node [right] at (p) {$0$};
\coordinate (p) at (0.5, 0.75) node [right] at (p) {$1$};
\coordinate (p) at (1, 0.75) node [right] at (p) {$0$};

\coordinate (p) at (-3, -0.75) node [left] at (p) {$1$};
\coordinate (p) at (-2.5, -0.75) node [left] at (p) {$1$};
\coordinate (p) at (-1, -0.75) node [left] at (p) {$1$};
\coordinate (p) at (1, -0.75) node [right] at (p) {$1$};
\coordinate (p) at (-2, -0.75) node [left] at (p) {$1$};
\coordinate (p) at (-1.5, -0.75) node [left] at (p) {$1$};
\coordinate (p) at (-1, -0.75) node [right] at (p) {$1$};
\coordinate (p) at (-0.5, -0.75) node [right] at (p) {$1$};
\coordinate (p) at (2, -0.75) node [right] at (p) {$0$};
\coordinate (p) at (0.5, -0.75) node [left] at (p) {$0$};
\coordinate (p) at (1, -0.75) node [left] at (p) {$0$};
\coordinate (p) at (2, -0.75) node [left] at (p) {$0$};
\coordinate (p) at (2.5, -0.75) node [right] at (p) {$1$};
\fill (a) circle (2pt);
\fill (b) circle (2pt);
\fill (c) circle (2pt);
\fill (d) circle (2pt);
\fill (e) circle (2pt);
\fill (f) circle (2pt);
\fill (g) circle (2pt);
\fill (h) circle (2pt);
\fill (i) circle (2pt);
\fill (j) circle (2pt);
\fill (k) circle (2pt);
\fill (l) circle (2pt);
\fill (m) circle (2pt);
\fill (n) circle (2pt);
\draw[very thin, blue](a) --  (b);
\draw[very thin, blue](a) -- (c);
\draw[thick,red](a) -- (d);
\draw[very thin, blue](a) -- (e);
\draw[thick, red](b) -- (f);
\draw[thick, red](b) -- (g);
\draw[thick, red](b) -- (j);
\draw[thick, red](b) -- (i);
\draw[thick, red](c) -- (h);
\draw[thick, red](c) -- (g);
\draw[thick, red](c) -- (j);
\draw[thick, red](c) -- (k);
\draw[very thin, blue](d) -- (n); 
\draw[very thin, blue](d) -- (m); 
\draw[very thin, blue](d) -- (j); 
\draw[very thin, blue](d) -- (k); 
\draw[thick, red](e) -- (l);
\draw[thick, red](e) -- (m);
\draw[thick, red](e) -- (j);
\draw[thick, red](e) -- (i);
\end{tikzpicture}
\caption{The weighting of the graph obtained from the image in Example \ref{eg1}. Edges with weight 0 are drawn thin (blue), and those with weight 1 is drawn thick (red). Not all numbers of weights are typed for visibility.}
\end{figure}
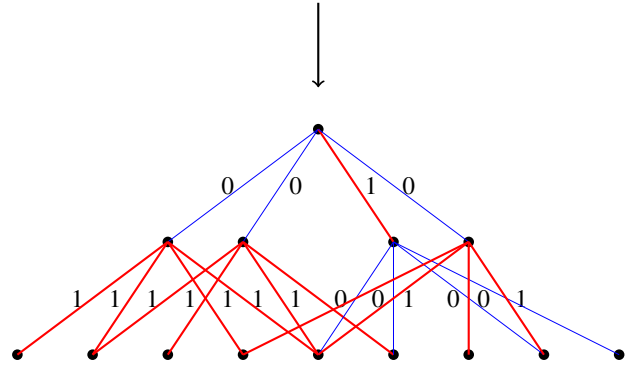
\end{ex} 
\begin{rem}
\label{generalization}
If we consider a $2^n \times 2^n$-pixel image and restrict vertices to those of the form $[2^n, 2^m, 2^k]\times [2^n, 2^m, 2^{k'}]$, then we obtain essentially the same object as a \textit{region quadtree}. One of the main differences between our construction and a quadtree is that the former has more information about the positional relationship among regions in the image. Hence the graph we construct has a quadtree as a subgraph. This property gives the graph a richer geometric structure, and it may extract more features of images than quadtrees do.
\end{rem}
We say that $d \colon X \times X \longrightarrow \mathbb{R}_{\geq 0}$ is a {\it pseudo-metric} on $X$ if the map $d$ satisfies the following:
\begin{enumerate}
\item $d(x, x) = 0$  for all $x \in X$,
\item $d(x, y) = d(y, x)$  for all $x, y \in X$,
\item $d(x, y) + d(y, z) \geq d(x, z)$ for all $x, y, z \in X$.
\end{enumerate}

By using the weight $w_f$, we define a pseudo-metric $d_f$ on $V_{M, N}$ as follows.
Take two vertices $A, B \in V_{M, N}$. 
For any path 
\[
P_{A,B} \colon A = A_0 - A_1 - \cdots - A_{n-1} - A_n = B
\]
from $A$ to $B$ in the graph $(V_{M, N}, E_{M, N})$, (that is, $\{A_i, A_{i+1}\} \in E_{M, N}$ for each integer  $0 \leq i < n$), we put 
\[
\tilde{w}_f(P_{A,B}) := \sum_{i=0}^{n-1} w_f(A_i, A_{i+1}). 
\]
We then define the pseudo-metric $d_f \colon V_{M,N} \times V_{M,N} \longrightarrow \mathbb{R}_{\geq 0}$ by 
\[
d_f(A, B) := \min_{P_{A,B}} \tilde{w}_f(P_{A,B}), 
\]
where $P_{A,B}$ runs over all the paths  from $A$ to $B$ in the graph $(V_{M,N}, E_{M,N})$. 
Note that, when $\{A, B\} \in E_{M,N}$, we have 
\[
d_f(A, B) = w_f(\{A,B\}). 
\]



\subsection{Construction of Vietoris-Rips complex}
We first briefly explain the notion of {\it simplicial complex}, which is a fundamental tool of algebraic topology.
 \cite{hatcher05} is a good reference that broadly covers this area. The notion of simplicial complexes is a generalization of polyhedra, and it consists of data of vertices, edges, faces, $\dots$, with inclusion relationship among them. 
 Every $n$-dimensional face can be reconstructed from $(n+1)$ points of its corners, hence the following definition is a good abstraction of our intuition.

\begin{df}[Simplicial complex]
Let $X$ be a set, and let $P(X)$ be its power set. A nonempty set $\emptyset \ne S \subset P(X)$ is a {\it simplicial complex} if it satisfies the following:
\[
B \subset A \in S \Rightarrow B \in S.
\]
We call elements of $S$ {\it simplices}, 
and we say that a simplex $B \in S$ is a {\it maximal simplex} if it satisfies the condition that 
\[
B \subset A \in S \Rightarrow B = A. 
\]
We call each simplex with $1$-element a {\it vertex}, each $2$-element simplex an {\it edge}, and each $3$-element simplex a {\it face}.
\end{df} 
\begin{ex}\label{exsc}
Let $X = \{a, b, c \}$ be a $3$-point set and let 
\[
\begin{cases}
S_0 := \{\emptyset, \{a\}, \{b \}, \{c\} \}, \\
S_1 := \{\emptyset, \{a\}, \{b \}, \{c\}, \{a, b\}, \{b, c\} \}, \\
S_2 := P(X) =  \{\emptyset, \{a\}, \{b \}, \{c\}, \{a, b\}, \{a, c\}, \{b, c\}, \{a, b, c\}\}.
\end{cases}
\]
Then $S_0$, $S_1$,  and $S_2$ are simplicial complexes corresponding to the following figures, respectively.
\begin{figure}[H]
\centering
\begin{tikzpicture}
\coordinate (a) at (0, 0) node [left]  {$a$} ;
\coordinate (b) at (2, 0) node [below] at (b) {$b$} ;
\coordinate (c) at (2, 1) node [above] at (c) {$c$};
\coordinate (d) at (1, -1) node [below] at (d) {$S_0$};
\fill (a) circle (2pt);
\fill (b) circle (2pt);
\fill (c) circle (2pt);

\end{tikzpicture}
\begin{tikzpicture}
\coordinate (a) at (0, 0) node [left]  {$a$} ;
\coordinate (b) at (2, 0) node [below] at (b) {$b$} ;
\coordinate (c) at (2, 1) node [above] at (c) {$c$};
\coordinate (d) at (1, -1) node [below] at (d) {$S_1$};
\fill (a) circle (2pt);
\fill (b) circle (2pt);
\fill (c) circle (2pt);

\draw[thick](b) --  (c);
\draw[thick](a) --  (b);
\end{tikzpicture}
\begin{tikzpicture}
\coordinate (a) at (0, 0) node [left]  {$a$} ;
\coordinate (b) at (2, 0) node [below] at (b) {$b$} ;
\coordinate (c) at (2, 1) node [above] at (c) {$c$};
\coordinate (d) at (1, -1) node [below] at (d) {$S_2$};
\fill[lightgray] (a)--(b)--(c)--cycle;
\fill (a) circle (2pt);
\fill (b) circle (2pt);
\fill (c) circle (2pt);

\draw[thick](b) --  (c);
\draw[thick](a) --  (c);
\draw[thick](a) --  (b);
\end{tikzpicture}
\caption{}
\end{figure}
\end{ex}

For any pseudo-metric space $X$, we can associate a simplicial complex that reflects both its metric structure and topology. See for example \cite{Carlsson09} for the detail.
\begin{df}[Vietoris-Rips complex]
{\it  The Vietoris-Rips complex} of a pseudo-metric space $(X, d)$ with a parameter $\varepsilon \geq 0$ is a simplicial complex $S_{\varepsilon} \subset P(X)$ defined as follows: 
\[
\{x_{0}, \dots, x_{n} \} \in S_{\varepsilon} \Leftrightarrow d(x_{i}, x_{j}) \leq \varepsilon \,\,\, \textrm{ for  all $i$ and $j$}. 
\]
\end{df}
\begin{ex}
Let $X = \{a, b, c\}$ be a 3-point set with the metric $d$ defined by 
\[
\begin{cases}
d(a, b) = 1, \\
d(b, c) = 1/2, \\
d(a, c) = \sqrt{5}/2.
\end{cases}
\]
Then the Vietoris-Rips complexes $S_0, S_1, S_2$ associated with $(X, d)$ and parameters $\varepsilon = 0, 1, 2$ are ones in Example \ref{exsc}.

\end{ex}

In the present paper, we consider the Vietoris-Rips complex $S(f)_{\varepsilon}$ associated with the pseudo-metric space $(V_{M, N}, d_f)$. For simplicity, we suppose that $M = N$. 

For a simplicial complex, we can compute its {\it homology groups}. See, for example, \cite{hatcher05} for the precise definition. Here we only define {\it 0-th homology}, that corresponds to the number of connected components of the given simplicial complex.

\begin{df}[0-th homology group]
For any simplicial complex $S$, we define an equivalence relation on the set of vertices of $S$ by 
\[
x\sim y \,\,\, \text{ if and only if }  \,\,\,  \{x, y\} \in S.
\]
Then the 0-th homology group of $S$ is defined to be the abelian group freely generated by the equivalence classes in $S\slash{\sim}$. 
Its rank, namely the number of equivalence classes, is called the number of connected components.
\end{df}

Especially, for any pseudo-metric space $X$, we obtain the family of homology groups of Vietoris-Rips complex parametrized by $\varepsilon \geq 0$. We call this {\it persistent homology} of $X$. 

\begin{ex}
When we consider the 0-th homology of simplicial complex, as mentioned above, it is sufficient to look at its subcomplex consisting of all simplices with at most two elements (so-called {\it 1-skelton}). 
Here we see the case of Vietoris-Rips complex $S(f)_{\varepsilon}$ obtained from the image $f \colon [3] \times [3] \longrightarrow [2]$ defined in Example~\ref{exgraph}. 
For any $0 \leq \varepsilon < 1$, we can see the connected components of $S(f)_{\varepsilon}$ by the following graph, which is obtained by deleting edges with weight 1 from the weighted graph in Example~\ref{exgraph}.
\begin{figure}[H]
    \centering
   \begin{tikzpicture}
\coordinate (a) at (0, 1.5)  ;
\coordinate (b) at (-2, 0)  ;
\coordinate (c) at (-1, 0) ;
\coordinate (d) at (1, 0) ;
\coordinate (e) at (2, 0) ;
\coordinate (f) at (-4, -1.5)  ;
\coordinate (g) at (-3, -1.5) ;
\coordinate (h) at (-2, -1.5) ;
\coordinate (i) at (-1, -1.5) ;
\coordinate (j) at (0, -1.5) ;
\coordinate (k) at (1, -1.5) ;
\coordinate (l) at (2, -1.5) ;
\coordinate (m) at (3, -1.5) ;
\coordinate (n) at (4, -1.5) ;

\fill (a) circle (2pt);
\fill (b) circle (2pt);
\fill (c) circle (2pt);
\fill (d) circle (2pt);
\fill (e) circle (2pt);
\fill (f) circle (2pt);
\fill (g) circle (2pt);
\fill (h) circle (2pt);
\fill (i) circle (2pt);
\fill (j) circle (2pt);
\fill (k) circle (2pt);
\fill (l) circle (2pt);
\fill (m) circle (2pt);
\fill (n) circle (2pt);

\draw[thick](a) --  (b);
\draw[thick](a) -- (c);

\draw[thick](a) -- (e);

\draw[thick](d) -- (n);
\draw[thick](d) -- (m);
\draw[thick](d) -- (j);
\draw[thick](d) -- (k);
\end{tikzpicture}
\caption{}
\end{figure}
\end{ex}
Hence the rank of the 0-th homology group of $S(f)_{\varepsilon}$, the number of connected components, is 7. 
For any $\varepsilon \geq 1$, the corresponding subcomplex is the whole graph, hence the rank of the 0-th homology group of $S(f)_{\varepsilon}$ is 1.

\subsection{Relation with depth of images}
\label{sec:relation_with_depth_of_images}
In the papers \cite{Asao19} and \cite{Asao19letter}, the first and the third authors defined an indicator of the complexity of images, that they call {\it depth} of images. 
In this subsection, we explain that the Vietoris-Rips complex $S(f)_\varepsilon$ covers the depth of an image $f$. 

We first recall the definition of the depth in our setting. 
For any 2-color image $f \colon [M]\times [M] \longrightarrow [2]$, 
we define the following function $\varphi_{d}$ as an index to measure how much color is non-uniformly distributed: 
\[
\varphi_{d}(f) := \min_{[M, d, \ast]\times [M, d, \ast] \in \Box_{M, M} }\left(\sum_{i,j \in [M, d, \ast]\times [M, d, \ast]}|f(i)-f(j)|\right). 
\]
In other words, the function $\varphi_{d}$ expresses the degree of bias of colors as a numerical value among the squares of size $d$.
If there is even one cell that is ``all white (black)'', its value will be 1. 
In particular, we have $\varphi_{1}(f) = 0$ for any image $f$.

\begin{df}[depth of images\cite{Asao19,Asao19letter}]
For any $M \times M$-pixel $C$-color image $f \colon [M]\times [M] \longrightarrow [C]$, 
we define the depth of $f$ by 
\[
\mathrm{depth}(f) := \frac{1}{M} \max\left\{d \in \{1, \ldots, M\} \ \middle| \ 
0 = \min_{p} 
\varphi_{d}(p \circ f) \right\},  
\]
where the map $p \colon f([M]\times [M]) \twoheadrightarrow \{0,1\}$ runs over the surjections from  $f([M]\times [M])$ to  $\{0,1\}$. 
\end{df}

\begin{prp}\label{prop:VR-depth}
Let $f \colon [M]\times [M] \longrightarrow [C]$ be an $M \times M$-pixel $C$-color image and let $S(f)_0$ be the  Vietoris-Rips complex associated with  the image $f$ and parameter $\varepsilon = 0$. 
Let $B \in S(f)_0$ be the simplex which contains the square $[M]\times [M]$ as a vertex. Then we have the following : 
\begin{itemize}
    \item[(1)]  The map $f$ is constant, that is, $f$ is a one-color image if and only if $\mathrm{depth}(f) = 1/M$. 
    \item[(2)] Suppose that $f$ is not constant. Let $d \in \mathbb{Z}_{\geq 0}$ denote the maximal size of squares in $\square_{M,M}$ that is NOT contained in $B$. 
    Then the value $M \cdot \mathrm{depth}(f)$ is equal to $d+1$. 
\end{itemize}
\end{prp}
\begin{proof}
Claim (1) follows immediately from the definition of $\mathrm{depth}(f)$. 
Let us  show claim (2). Suppose that $f$ is not constant. Then $B \neq \Box_{M,M}$, and hence one can consider the maximal size of squares that is not contained in $B$. We also note that $M \cdot \mathrm{depth}(f) \geq 2$. We put 
 \[
 e := M \cdot \mathrm{depth}(f) - 1 \geq 1. 
 \]
 By the definition of $\mathrm{depth}(f)$, any square of size at least $e+1$ is contained in $B$. This fact implies $d \leq e$. 
 Since $M \cdot \mathrm{depth}(f) > e$, there is a square $\Box_1 \in \Box_{M,M}$ of size $e$  such that $f(\Box_1) \neq f([M] \times [M])$. 
 Then $\Box_1 \not\in B$, and so $e \leq d$. We conclude that 
 \[
 M \cdot \mathrm{depth}(f) = e + 1 = d + 1. 
 \]
\end{proof}

As we see in Proposition \ref{prop:VR-depth}(1), the depth of any one-color image is $1/M$. 
Conversely, the depth of a complex image shown on the right in Figure \ref{apples} is very large. 
In other words, the depth of an image is an index for measuring the complexity of the given image. 
Meaningful images are expected to have a slightly large depth because the colors are considered to be biased to some extent. 
In Figure \ref{giraf}, the value $\left(\sum_{i,j \in \Box}|f(i)-f(j)|\right)$ is represented by the color strength when $d$ is gradually increased. 
As $d$ increases, the cells become finer. When a black cell appears for the first time at time $d$, then $d$ corresponds to the depth. 
\begin{figure}[H]
\begin{center}
\includegraphics[width=0.4\linewidth]{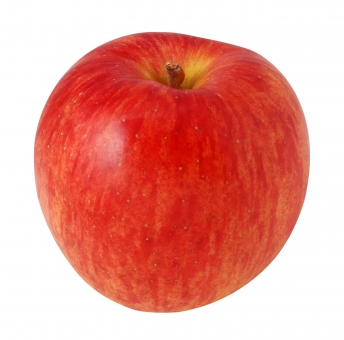}
\includegraphics[width=0.4\linewidth]{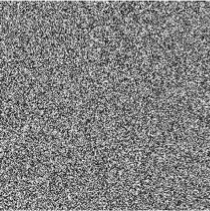}
\caption{Right has very large depth, and left is intermediate.}
\label{apples}
\end{center}
\end{figure}
\begin{figure}[htb]
\begin{center}
\includegraphics[width=\linewidth]{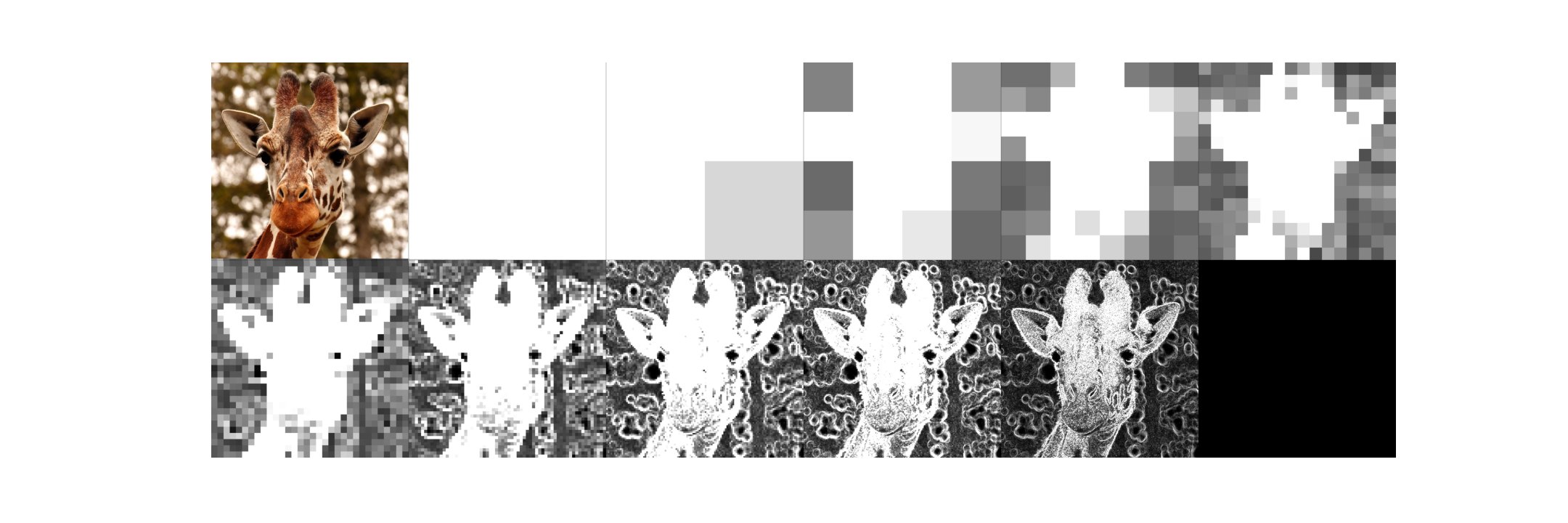}
\end{center}
\caption{The depth gets larger as it goes to right bottom from upper left. The top left image is the original image.}
\label{giraf}
\end{figure}

In the papers \cite{Asao19} and \cite{Asao19letter}, the first and the third authors show some interesting asymptotic behaviors of depth as $M \to \infty$. For example, it is shown in \cite{Asao19} and \cite{Asao19letter} that almost all images have very large depth, which means 
almost all images are
too complex to have visual information. This result suggests us a possibility that computer can classify images into meaningful or -less ones by calculating depth.

As we see in Proposition \ref{prop:VR-depth}(2), our Vietoris-Rips complex obtained from an image contains information about the depth and thus also the complexity of the image. Hence it would be worth investing this object for image analysis.


\section{Connected components and maximal simplices of the Vietoris-Rips complex}
\label{section:connected-component}

In this section, we explain the relationship between the Vietoris-Rips complex constructed in the previous section and the original image. In particular, we note that all squares belonging to the same simplex have the same number of colors admitting error controlled by the parameter $\varepsilon$. Clearly, it corresponds to considering the 0-th persistent homology. We first define a feature {\it information concentration}, and we note that a square with minimal size among those in the same simplex has high information concentration. This leads to a method of object detection that we actualize in \S~\ref{section:experiment}.

\begin{df}
Let $f \colon [M]\times [N] \longrightarrow [C]$ be an image. 
For a subset $D \subset [M]\times [N]$, we define the {\it information concentration} of $D$ by the ratio $\#f(D)/\# D$, where the numerator is the number of colors of $f$ restricted to $D$, and the denominator is the area of $D$.
\end{df}

By definition, it quantifies the extent to which the information of an entire image is locally concentrated. Hence detecting areas with high information concentration are candidates of areas containing salient objects, or their boundaries. Now we explain the properties of the Vietoris-Rips complex we construct.


\subsection{$\varepsilon = 0$ case}
In case $\varepsilon = 0$, each maximal simplex of the Vietoris-Rips complex obtained from the pseudo-metric space $(V_{M, M}, d_f)$ corresponds to each connected component of the graph obtained from the weighted graph $(V_{M, M}, E_{M, M}, w_f)$ by deleting edges with positive weight. It is easily seen that all squares belonging to the same connected component have the same number of colors. Hence the smaller size the square has, the higher concentration of information it has relative to the other squares in the component. This observation leads us to consider that information of the image concentrates on the square with the minimal size in the component which has a large number of colors. In particular, the component containing the vertex $[M, M, 0]\times [M, M, 0]$ which corresponds to the whole of the image has much information.


\subsection{$\varepsilon > 0$ case}\label{positive}
In case $\varepsilon > 0$, similar to the previous case, all squares belonging to the same simplex have the same number of colors admitting error $\varepsilon$. 
Hence the smaller size the square has, the higher concentration of information it has relative to the other squares in the component. 
However, it is difficult to describe explicitly or find the maximal simplices unlike the case of $\varepsilon = 0$.
(It reduces to solve the clique problem for a graph with $O(M^3)$ vertices.) 
On the other hand, for the purpose of detecting the domain on which information concentrates, it is sufficient to find squares with minimal size in the maximal simplex containing the vertex $[M, M, 0]\times [M, M, 0]$ which corresponds to the whole image. 
We can achieve it by finding the minimal size squares which have at least $c-\varepsilon$ colors, where we suppose that the whole image has $c$ colors. 
This process is independent of the choice of pseudo-metric $d$, and contains the method in the case of $\varepsilon = 0$.

\subsection{Robustness to noise}
As mentioned in \S \ref{positive}, the parameter $\varepsilon$ of Vietoris-Rips complex can be regarded as an error we can admit for counting the number of colors. Hence any noise on the image can be neglected for detecting salient objects by enlarging the parameter $\varepsilon$ appropriately.



\section{Numerical experiments}
\label{section:experiment}
In this section, we apply the method explained in \S \ref{section:connected-component} to the salient object detection. We first explain our algorithm in \S \ref{algorithm}, then show results in \S \ref{experiments}.
We find minimal size vertices belonging to the connected component containing the vertex corresponding to the whole image.  The squares corresponding to these vertices are expected to contain salient objects as explained in \S \ref{section:connected-component}. Fluctuation of the parameter $\varepsilon$ produces robustness to the noise of the image. 

\subsection{Algorithms}\label{algorithm}
Before we start explaining our algorithm, we introduce some notations. 
For notational simplicity, we put 
\[
v_{k,i,j} := [N,N-k,i]\times[N,N-k,j] \in  V_{N,N}. 
\]
Note that 
\[
V_{N,N} = 
\left\{ v_{k,i,j} \,\, \middle| 
\begin{array}l
0 \leq k \leq N-1, 
\\
0 \leq i \leq k, 
\\
0 \leq j \leq k
\end{array}
\right\}. 
\]
We index vertices $v \in V_{N,N}$ by the map $g \colon V_{N,N}\longrightarrow \mathbb{Z}_{\geq 0}$ defined as 
\[
    g(v_{k,i,j}
    ) := (k+1)i+j+\sum_{l=0}^{k}l^2.
\]
If we consider the lexicographical order on the set $V_{N,N}$ with respect to the index $(k,i,j)$ of $v_{k,i,j}$, the map $g \colon V_{N,N} \longrightarrow \mathbb{Z}_{\geq 0}$ is an order preserving injection. 

Note that, the larger size a vertex has, the smaller index it is given by $g$ (see Figure \ref{fig:index}). 
\begin{figure}[H]
\begin{center}
\begin{tikzpicture}
\draw[thick][red] (0,0) rectangle (2,2) ;
\draw (1,1)[red] node{{\Huge 0}};
\end{tikzpicture}
\begin{tikzpicture}
\draw[thick][gray] (0,0) rectangle (2,2) ;
\draw[thick][gray](0, 0.66) -- (2, 0.66);
\draw[thick][gray](1.33, 1.33) -- (2, 1.33);
\draw[thick][gray](0.66, 0) -- (0.66, 0.66);
\draw[thick][gray](1.33, 0) -- (1.33, 2);
\draw[thick][red] (0,0.66) rectangle (1.33,2) ;
\draw (0.66,1.33)[red] node{{\Huge 1}};
\end{tikzpicture}
\begin{tikzpicture}
\draw[thick][gray] (0,0) rectangle (2,2) ;
\draw[thick][gray](0, 1.33) -- (0.66, 1.33);
\draw[thick][gray](0, 0.66) -- (0.66, 0.66);
\draw[thick][gray](0.66, 0) -- (0.66, 0.66);
\draw[thick][gray](1.33, 0) -- (1.33, 0.66);
\draw[thick][red] (0.66,0.66) rectangle (2,2) ;
\draw (1.33,1.33)[red] node{{\Huge 2}};
\end{tikzpicture}

\vspace{0.3cm}

\begin{tikzpicture}
\draw[thick][gray] (0,0) rectangle (2,2) ;
\draw[thick][gray](1.33, 0.66) -- (2, 0.66);
\draw[thick][gray](0, 1.33) -- (2, 1.33);
\draw[thick][gray](0.66, 1.33) -- (0.66, 2);
\draw[thick][gray](1.33, 0) -- (1.33, 2);
\draw[thick][red] (0,0) rectangle (1.33,1.33) ;
\draw (0.66,0.66)[red] node{{\Huge 3}};
\end{tikzpicture}
\begin{tikzpicture}
\draw[thick][gray] (0,0) rectangle (2,2) ;
\draw[thick][gray](0, 0.66) -- (0.66, 0.66);
\draw[thick][gray](0, 1.33) -- (2, 1.33);
\draw[thick][gray](0.66, 1.33) -- (0.66, 2);
\draw[thick][gray](1.33, 1.33) -- (1.33, 2);
\draw[thick][red] (0.66,0) rectangle (2,1.33) ;
\draw (1.33, 0.66)[red] node{{\Huge 4}};
\end{tikzpicture}
\begin{tikzpicture}
\draw[thick][red] (-0.1,-0.1) rectangle (0.56,0.56) ;
\draw[thick][red] (0.66,0.66) rectangle (1.33,1.33) ;
\draw[thick][red] (1.43,1.43) rectangle (2.1,2.1) ;
\draw[thick][red] (0.66,-0.1) rectangle (1.33,0.56) ;
\draw[thick][red] (1.43,-0.1) rectangle (2.1,0.56) ;
\draw[thick][red] (-0.1,0.66) rectangle (0.56,1.33) ;
\draw[thick][red] (-0.1,1.43) rectangle (0.56,2.1) ;
\draw[thick][red] (0.66, 1.43) rectangle (1.33, 2.1) ;
\draw[thick][red] (1.43, 0.66) rectangle (2.1, 1.33) ;
\draw (0.23, 0.23)[red] node{{\Large 11}};
\draw (1, 1)[red] node{{\Large 9}};
\draw (1.76, 1.76)[red] node{{\Large 7}};
\draw (1, 0.23)[red] node{{\Large 12}};
\draw (1.76, 0.23)[red] node{{\Large 13}};
\draw (0.23, 1)[red] node{{\Large 8}};
\draw (0.23, 1.76)[red] node{{\Large 5}};
\draw (1, 1.76)[red] node{{\Large 6}};
\draw (1.76, 1)[red] node{{\Large 10}};
\end{tikzpicture}
\caption{The index $g$ for $N = 3$.}
\label{fig:index}
\end{center}
\end{figure}
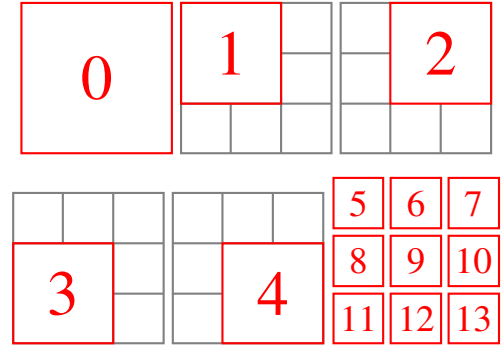
\begin{ex}
\label{ex:vertex}
Let us consider the case that $N=3$. 
For $k=0$, the only square of size $3 \times 3$ is 
\begin{align*}
    [3,3,0]\times[3,3,0],
\end{align*}
and it has the index $0$.
For $k=1$, there are exactly 4 squares of size $2 \times 2$. They are 
\begin{align*}
    &[3,2,0]\times[3,2,0], \,\,\, [3,2,0]\times[3,2,1], \\
    &[3,2,1]\times[3,2,0], \,\,\, [3,2,1]\times[3,2,1],
\end{align*}
and have indices $1, 2, 3, 4$, respectively.
For $k=2$, there are exactly 9 squares of size $1 \times 1$. They are 
\begin{align*}
    &[3,1,0]\times[3,1,0], \,\,\, [3,1,0]\times[3,1,1], \,\,\, [3,1,0]\times[3,1,2], \\
    &[3,1,1]\times[3,1,0], \,\,\, [3,1,1]\times[3,1,1], \,\,\, [3,1,1]\times[3,1,2], \\
    &[3,1,2]\times[3,1,0], \,\,\, [3,1,2]\times[3,1,1], \,\,\, [3,1,2]\times[3,1,2],
\end{align*}
and have indices $5, 6, 7, 8, 9, 10, 11, 12, 13$, respectively.
\end{ex}
With the above preparations, now we explain the algorithm \ref{algorithm:1-skelton}. 
This algorithm outputs three lists $V_{\mathrm{from}}$, $V_{\mathrm{to}}$ and $W$ from an $N\times N$-pixel image as an input. The lists $V_{\mathrm{from}}$ and $V_{\mathrm{to}}$ store the sources and targets of edges $E_{N, N}$ respectively with appropriate orientations and in an appropriate order explained below. 
The list $W$ corresponds to the adjacent matrix of the weighted graph $(V_{N, N}, E_{N, N}, w_{f})$, while the weights are stored in the order explained as follows.


By orienting the edges by the inclusion relationship between the vertices they connect, we regard $E_{N,N} \subset V_{N,N} \times V_{N,N}$, that is, 
\[
E_{N,N} = 
\left\{ (B, A) \in  V_{N,N} \times V_{N,N} \,\, \middle| 
\begin{array}l
A \subset B,   
\\
\textrm{the size of $B$}  = 1 + \textrm{the size of $A$} 
\end{array}
\right\}. 
\]
We denote the projection $E_{N, N} \longrightarrow V_{N, N}$ to the $i$-th coordinate by $\mathrm{pr}_i$ for $i = 1, 2$. Note that 
\begin{align*}
    \mathrm{pr}_1^{-1}(v_{k,i,j}) = \{(v_{k,i,j}, &v_{k+1,i,j}), \, (v_{k,i,j}, v_{k+1,i,j+1}), 
\\
&(v_{k,i,j}, v_{k+1,i+1,j}), \, (v_{k,i,j}, v_{k+1,i+1,j+1}) \}.  
\end{align*}
Namely, there are exactly four oriented edges which start from $v_{k,i,j}$, and their targets are $v_{k+1,i,j}$, $v_{k+1,i,j+1}$, $v_{k+1,i+1,j}$, and $v_{k+1,i+1,j+1}$. 
We define a map $h \colon E_{N,N} \longrightarrow \mathbb{Z}_{\geq 0}$ by 
\begin{align*}
    h((v_{k,i,j}, v_{k+1,i,j})) &:= 4g(v_{k,i,j}), 
    \\
    h((v_{k,i,j}, v_{k+1,i,j+1})) &:= 4g(v_{k,i,j}) + 1, 
    \\
    h((v_{k,i,j}, v_{k+1,i+1,j})) &:= 4g(v_{k,i,j}) + 2, 
    \\
    h((v_{k,i,j}, v_{k+1,i+1,j+1})) &:= 4g(v_{k,i,j}) + 3.  
\end{align*}
If we consider the lexicographical order on the set $E_{N,N} \subset V_{N,N} \times V_{N,N}$ with respect to the order induced by the index $g$, the map $h \colon E_{N,N} \longrightarrow \mathbb{Z}_{\geq 0}$ is an order preserving injection. 
By using $h$, we can define the lists $V_{\mathrm{from}}$, $V_{\mathrm{to}}$ and $W$ as follows: 
\begin{align*}
    V_{\mathrm{from}}[s] &= g(\mathrm{pr}_1(h^{-1}(s))), \\
    V_{\mathrm{to}}[s] &= g(\mathrm{pr}_2(h^{-1}(s))), \\
    W[s] &= w_{f}(h^{-1}(s)).
\end{align*}
The way to represent a graph, as above, by a list of sources of edges $V_{\mathrm{from}}$, a list of targets of edges $V_{\mathrm{from}}$, and a list of weights on edges $W$ is called {\it Compressed Sparse Row (CSR) form}. It is why we set the above three lists that we can deal with graphs compressed in the CSR form very fast in the Scipy library of Python.

\begin{algorithm}[h]
\caption{Construction of the 1-skelton of the Vietoris-Rips complex using all squares}
 \label{algorithm:1-skelton}
 \begin{algorithmic}
 \REQUIRE $f$ is a $C$-color $N \times N$-pixel image. $\varepsilon$ is a non-negative real number.
    \ENSURE $V_{\mathrm{from}}$, $V_{\mathrm{to}}$, $W$, $\widetilde{W}$. 
    \STATE $V_{\mathrm{from}}$, $V_{\mathrm{to}}$, $W$, $\widetilde{W}$ are empty lists. 
    \FOR{$k = 0$ to $N-1$}
        \FOR{$i = 0$ to $k$} 
            \FOR{$j = 0$ to $k$}
            \STATE Assign $g(v_{k, i, j})$ to list $V_{\mathrm{from}}$.
            \STATE Assign $g(v_{k+1, i, j})$ to list $V_{\mathrm{to}}$.
            \STATE Assign $\#f(v_{k, i, j}) - \#f(v_{k+1, i, j})$ to list $W$.
            \STATE Assign $g(v_{k, i, j})$ to list $V_{\mathrm{from}}$.
            \STATE Assign $g(v_{k+1, i, j+1})$ to list $V_{\mathrm{to}}$.
            \STATE Assign $\#f(v_{k, i, j}) - \#f(v_{k+1, i, j+1})$ to list $W$.
            \STATE Assign $g(v_{k, i, j})$ to list $V_{\mathrm{from}}$.
            \STATE Assign $g(v_{k+1, i+1, j})$ to list $V_{\mathrm{to}}$.
            \STATE Assign $\#f(v_{k,i,j}) - \#f(v_{k+1,i+1,j})$ to list $W$.
            \STATE Assign $g(v_{k, i, j})$ to list $V_{\mathrm{from}}$.
            \STATE Assign $g(v_{k+1, i+1, j+1})$ to list $V_{\mathrm{to}}$.
            \STATE Assign $\#f(v_{k,i,j}) - \#f(v_{k+1,i+1,j+1})$ to list $W$.
            \ENDFOR
        \ENDFOR
    \ENDFOR
    \FOR{$s=0$ to $\#W - 1$}
    \IF{$W[s] > \varepsilon$}
    \STATE Assign $0$ to list $\widetilde{W}$
    \ELSE
    \STATE Assign $1$ to list $\widetilde{W}$
    \ENDIF
    \ENDFOR
    \end{algorithmic}
\end{algorithm}

The connected components of the 1-skeleton obtained from the algorithm \ref{algorithm:1-skelton} can be indexed by using \textit{sparse.csgraph.connected\_components} method in the Scipy library. 
This method outputs a list ${\mathrm CC}$ of indexed connected components from an input of a graph compressed in the CSR form.
The index of components is given by referring to the smallest index of vertices which belong to the component. 
The index of the component containing the vertex with index 0 is 0. This component and the vertices with maximal indices contained in this component have good information for object detection as explained in \S \ref{section:connected-component}. 
We conduct object detection by highlighting squares whose corresponding vertices are contained in the index 0 component $C_0$, and have minimal size among $C_0$. To find all minimal size squares in $C_0$, we first choose a vertex with maximum index, namely, $v_{\mathrm{tmp}}$. Then we check each vertex with the same size as $v_{\mathrm{tmp}}$ if it is contained in $C_0$ or not. We highlight the corresponding square when the vertex is contained. The algorithm \ref{algorithm:detection} is a pseudo-code.
\begin{algorithm}[h]
\caption{Highlighting minimal size squares}
\label{algorithm:detection}
\begin{algorithmic}
\REQUIRE $f$ is a $C$-color $N \times N$-pixel image. $V_{\mathrm{from}}$, $V_{\mathrm{to}}$, $\widetilde{W}$ are lists obtained by the algorithm \ref{algorithm:1-skelton}.
\ENSURE  $d$ is the detection result of $f$. 
\STATE $d$ is a $1$-color $N \times N$-pixel image. 
\STATE $v_{\mathrm{tmp}}, k_{\mathrm{tmp}}, k' \leftarrow 0$. 
\STATE Apply \textit{csgraph.connected\_components} to $V_{\mathrm{from}}$, $V_{\mathrm{to}}$ and $\widetilde{W}$ to obtain $\mathrm{CC}$. 
\FOR{$s=0$ to $\# \mathrm{CC} -1$}
    \IF{$\mathrm{CC}[s] = 0$}
        \STATE $v_{\mathrm{tmp}} \leftarrow s$
    \ENDIF
\ENDFOR
\STATE $k_\mathrm{tmp} \leftarrow \text{size of the square $v_{\mathrm{tmp}}$}$
\STATE $k' \leftarrow N-k_\mathrm{tmp}$
\FOR{$i=0$ to $k'$}
    \FOR{$j=0$ to $k'$}
    \IF{$\mathrm{CC}[g(v_{k',i,j})] = 0$}
        \STATE $d|_{v_{k',i,j}} \leftarrow f|_{v_{k',i,j}}$
    \ENDIF
    \ENDFOR
\ENDFOR
\end{algorithmic}
\end{algorithm}

\subsection{Experiments}\label{experiments}
In the following, we apply the method explained in \S \ref{section:connected-component} to the salient  object  detection for the following two $100\times 100$-images (Figure \ref{fig2}). They are pictures of \textbf{(a)} a Japanese tea and a sweet and \textbf{(b)} a squirrel.
In this experiment, the parameter $\varepsilon$ is set to 24 levels, starting from 10 and going up to 240 in increments of 10. As written in the pseudo-code \ref{algorithm:detection}, there is room for a choice of a class of domains in our process. We used two classes in this experiment, all squares and all rectangles.
The results with a class of squares, which is described in \S 4, can be found in \S 5.1, along with a comparison of the detection results using the second minimal squares in addition to the minimal squares. 
The results with a class of rectangles can be found in \S 5.2. We first expected that the process using rectangles improve the accuracy of the detection. However, we checked that there is no great difference despite that the rectangle case computationally costs more than the square case. 
While this series of experiments were conducted on grayscale images, the same method can be applied to general multi-channel images by appropriately setting weights on edges.


\begin{figure}[htb]
\centering
\subfigure[Japanese tea and a sweet]{
\includegraphics[width=0.4\linewidth]{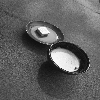}
}
\subfigure[Squirrel]{
\includegraphics[width=0.4\linewidth]{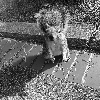}
}
\caption{$M\times M$-images of tea with a tea sweet, and squirrel ($M = 100$). The numbers of colors are $240, 254$ respectively.}
\label{fig2}
\end{figure}


\subsubsection{Process using all squares}
In this subsection, we describe the results of experiments using a class of all square domains. 
The process that mainly occupies the calculation time is the step of counting all the squares for constructing the graph, and it takes $O(M^{3})$ for an $M\times M$-image (the order of square sum from 1 to $M$). 


\begin{figure}[htb]
\centering
\subfigure[Japanese tea and a sweet]{
\includegraphics[width=\linewidth]{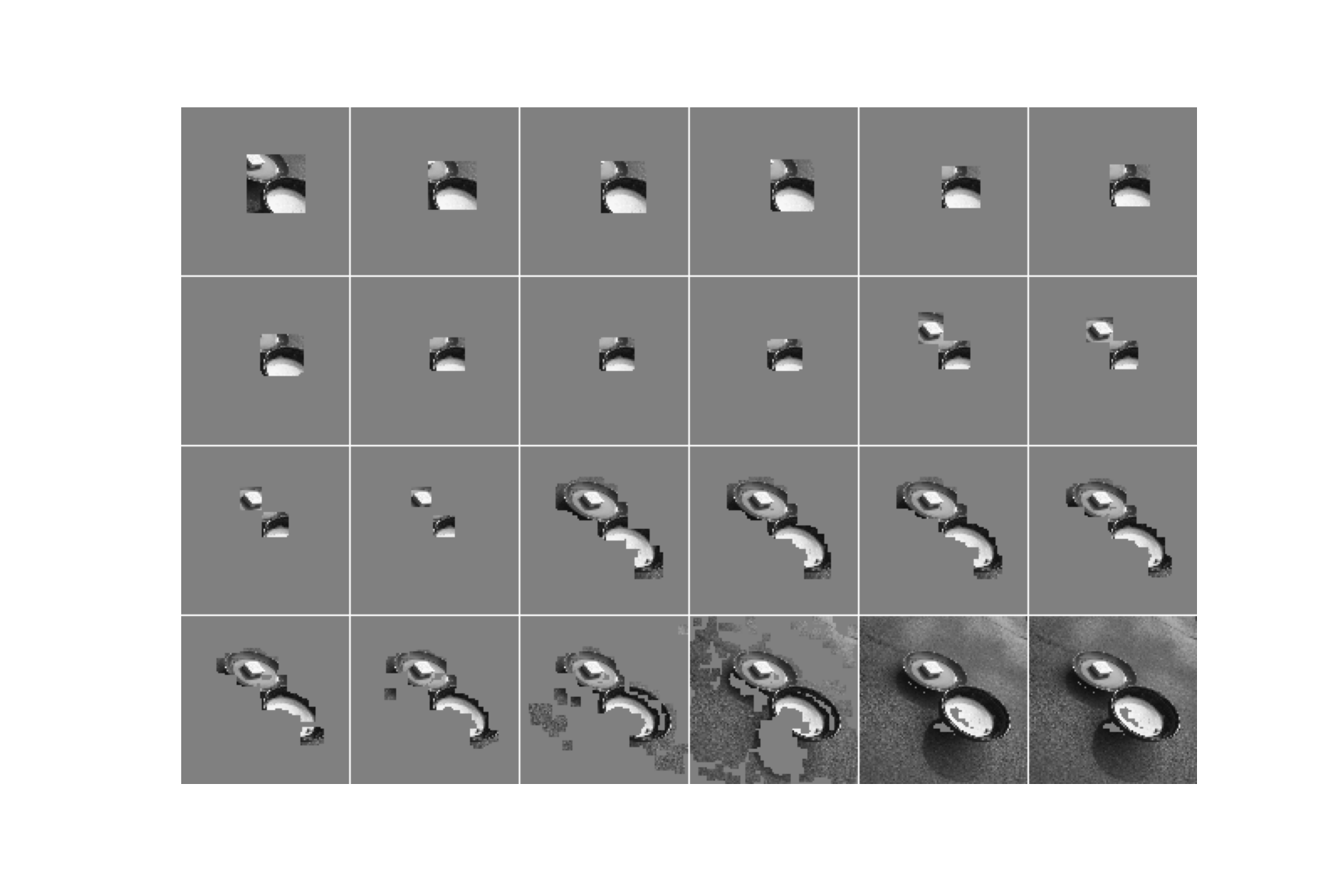}
}
\subfigure[Squirrel]{
\includegraphics[width=\linewidth]{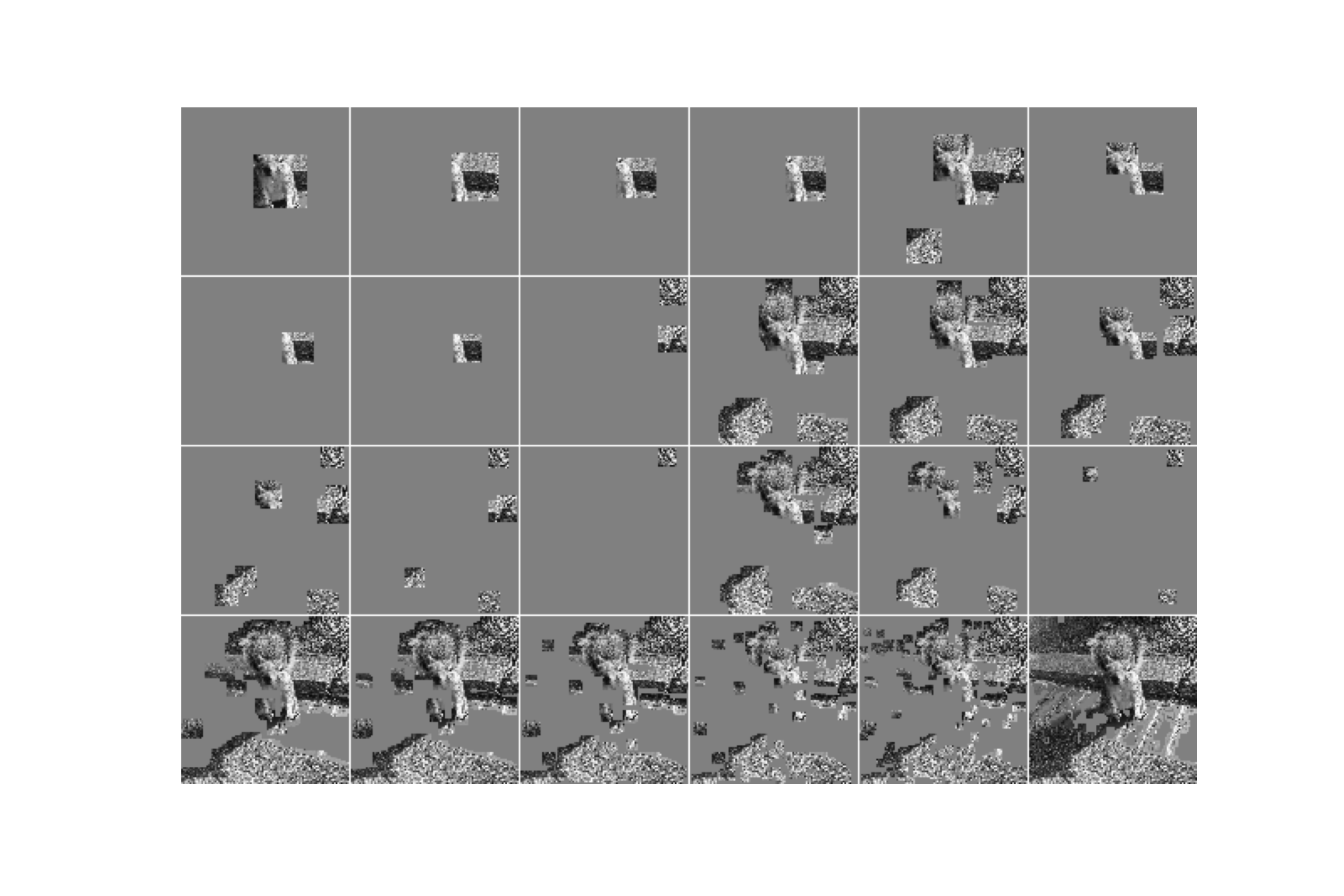}
}
\caption{$c = 240$, $c = 254$ respectively from left. $n = 0, 1, 2, \ldots$ from upper left to bottom right.}
\label{fig3}
\end{figure}
Figure \ref{fig3} represents the squares in $c$-color images with minimal size that uses at least $c-n\varepsilon \  ( n = 0, 1, 2, \ldots)$ colors for $\varepsilon = 10$ and $c = 240, 254$. The actual detected area is shown in the original image, while the other areas are filled in gray. 
By adjusting the error parameter $\varepsilon$ appropriately, we can see that regions with high concentration of information certainly detect the salient object. $n = 14$ appears to detect the main object region in the image. 
Figure \ref{fig4} shows the cumulative display of the regions detected in Figure \ref{fig3} up to $n$-th for visibility. 

\begin{figure}[htb]
\centering
\subfigure[Japanese tea and a sweet]{
\includegraphics[width=\linewidth]{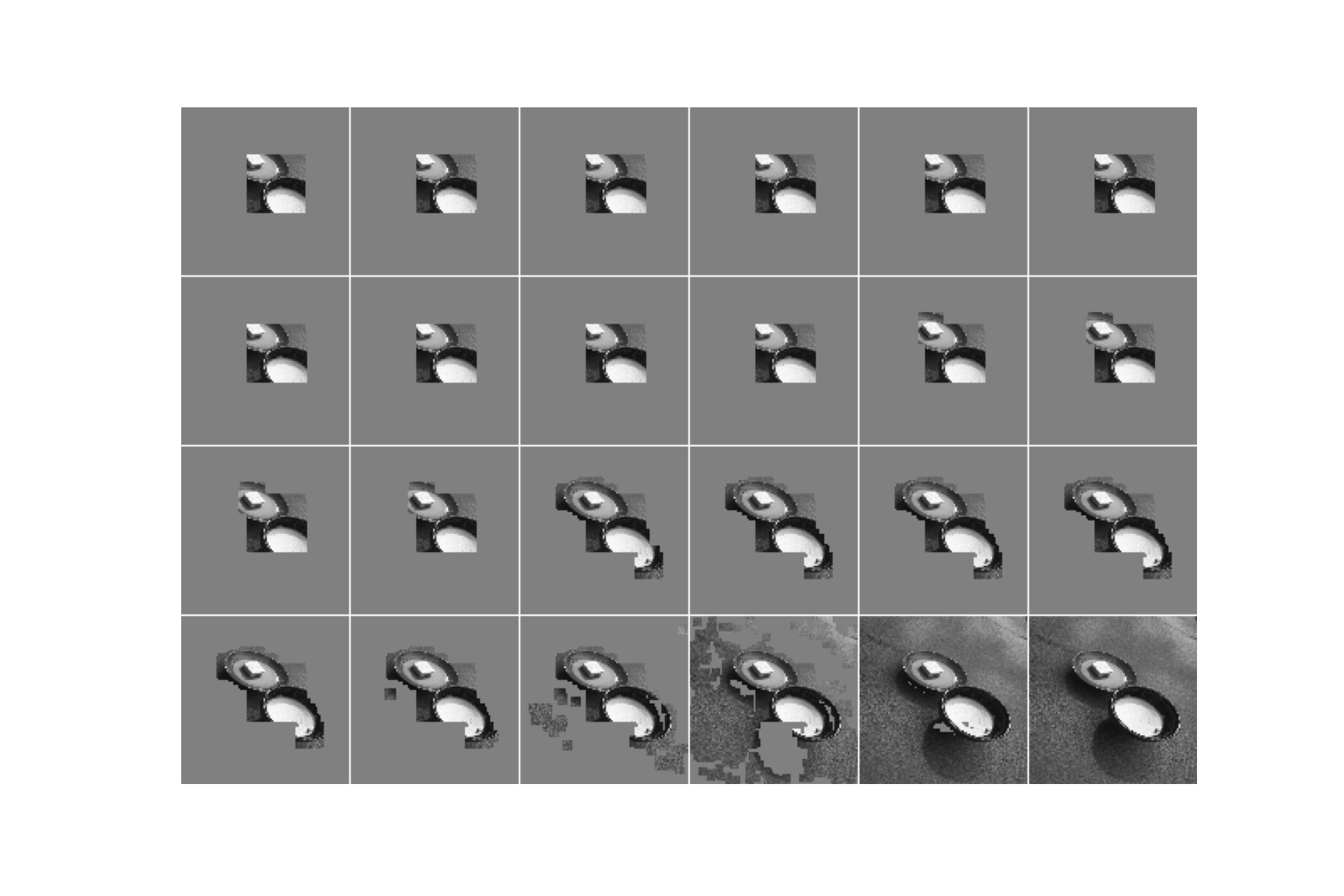}
}
\subfigure[Squirrel]{
\includegraphics[width=\linewidth]{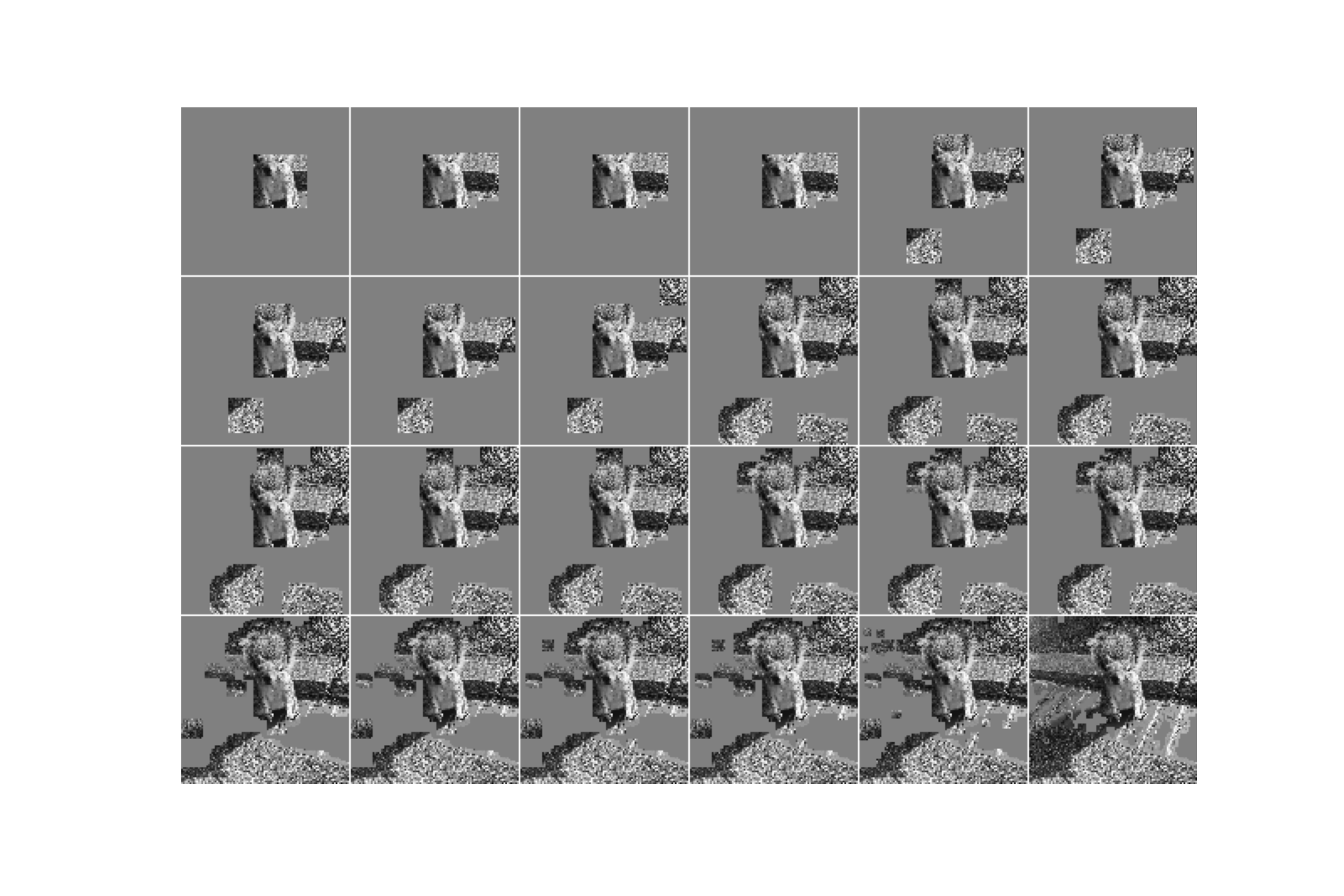}
}
\caption{the cumulative re-display of the $n$-th region detected in Fugure \ref{fig3}.}
\label{fig4}
\end{figure}

Figure \ref{fig5} represents images from Figure \ref{fig4} that seems detecting objects well. Such choosing process corresponds to a process of picking up the image whose gray regions rapidly decrease. 

\begin{figure}[htb]
\centering
\subfigure[Japanese tea and a sweet]{
\includegraphics[width=0.4\linewidth]{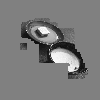}
}
\subfigure[Squirrel]{
\includegraphics[width=0.4\linewidth]{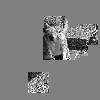}
}
\caption{$n = 19$, $n = 4$ from left.}
\label{fig5}
\end{figure}
Figure \ref{fig6} represents the results of a similar process that uses both squares with minimal and the next minimal size. It seems that detection is done better. 

\begin{figure}[htb]
\centering
\subfigure[Japanese tea and a sweet]{
\includegraphics[width=\linewidth]{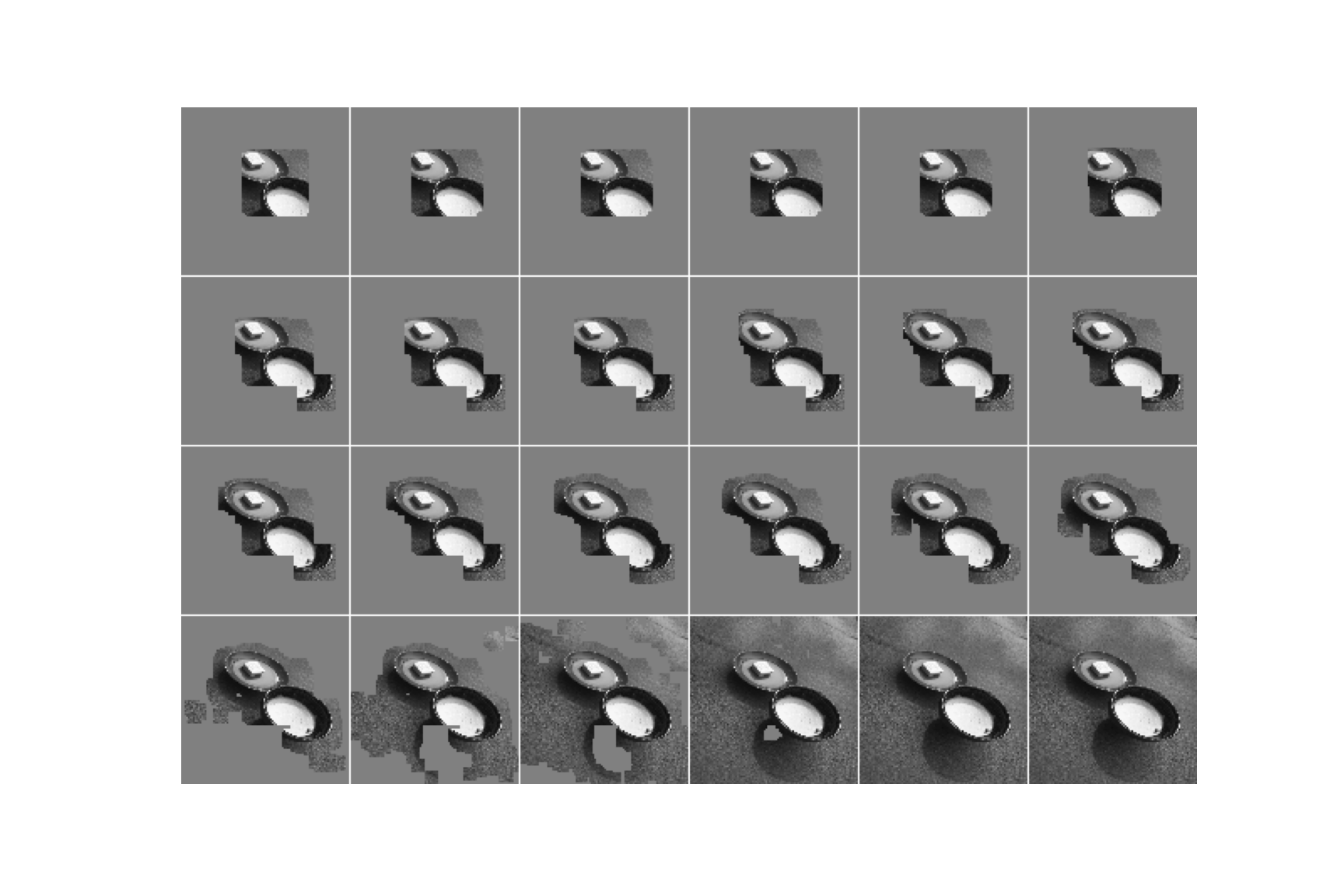}
}
\subfigure[Squirrel]{
\includegraphics[width=\linewidth]{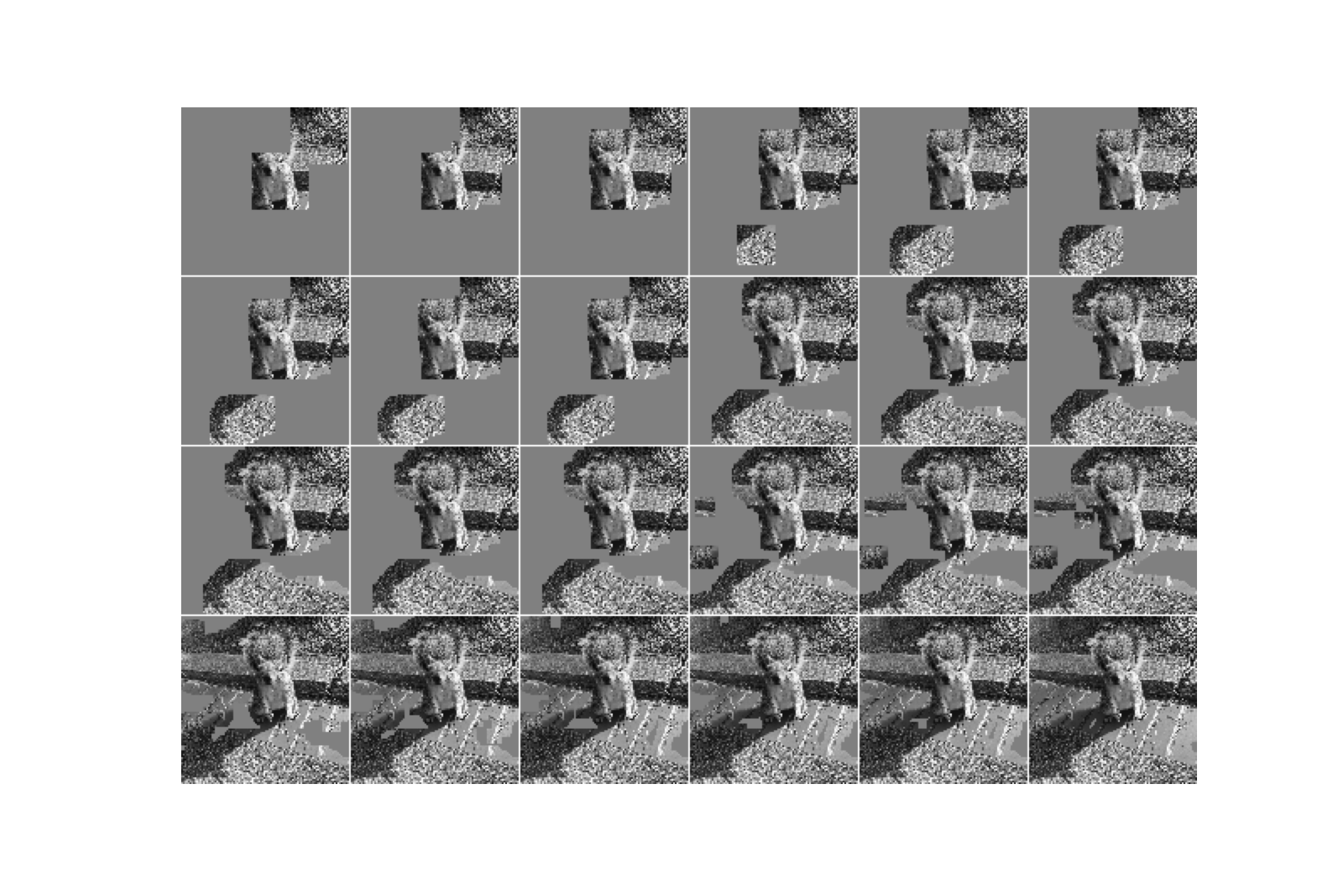}
}
\caption{Representing cumulatively the squares with minimal and next minimal.}
\label{fig6}
\end{figure}
%
%
%
%
%
\subsubsection{Process using all rectangles}

In the previous subsection, we discussed object detection in the case of using squares, but we can also consider regions other than squares as subregions of the image.
In this subsection, we will examine the results when using rectangles.
The computational cost of using a square was $O(M^3)$, but the computational cost of using a rectangle is $O(M^4)$ (the order of the square of the sum from 1 to $M$).
Since counting all the rectangles would not only take a lot of time but also the accuracy would not be improved by using extremely long and narrow rectangles for detection, we experimented with rectangles with aspect ratios from $1/3$ to $3$. The result is shown in Figure \ref{fig7}. 



\begin{figure}[htb]
\centering
\subfigure[Japanese tea and a sweet]{
\includegraphics[width=\linewidth]{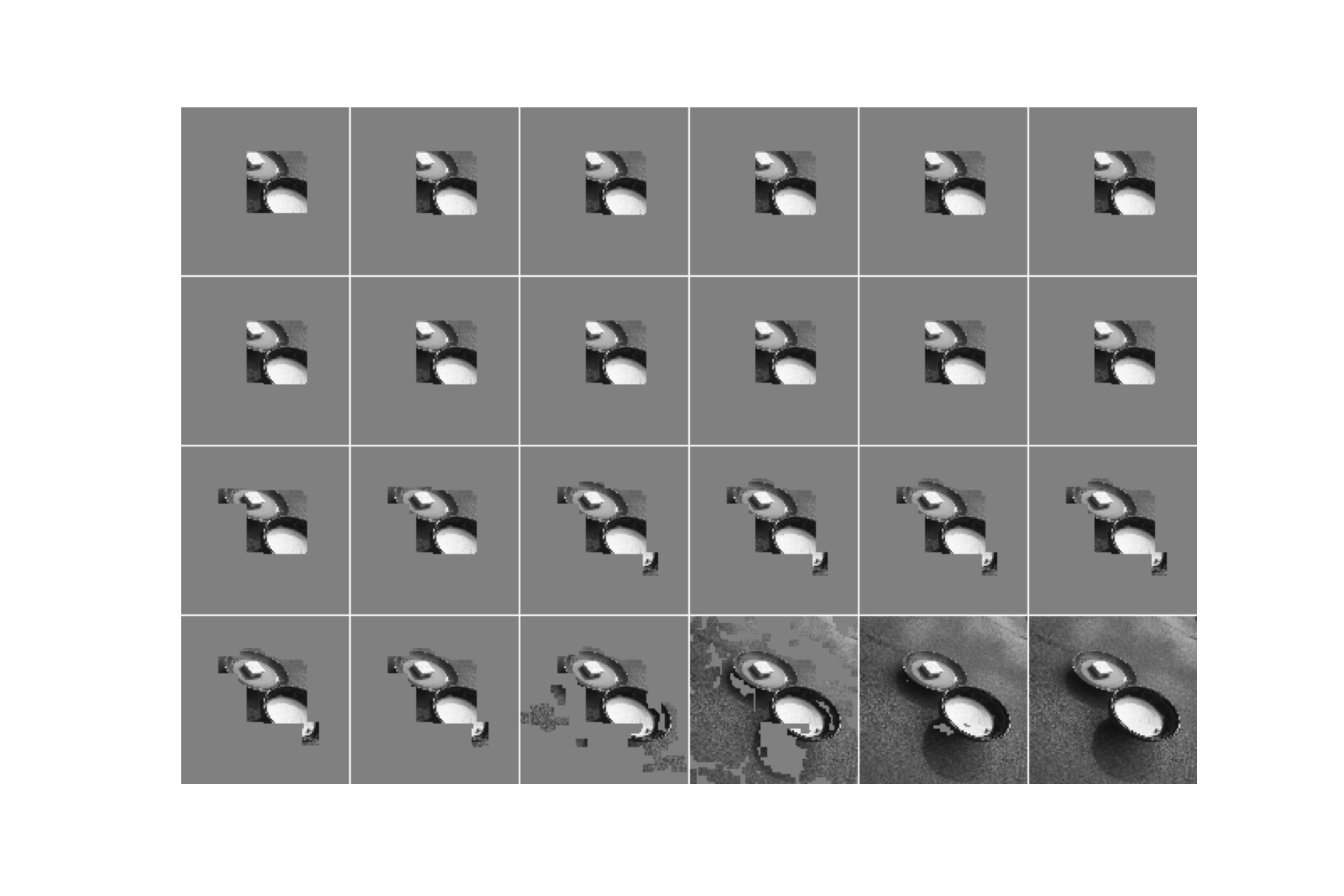}
}
\subfigure[Squirrel]{
\includegraphics[width=\linewidth]{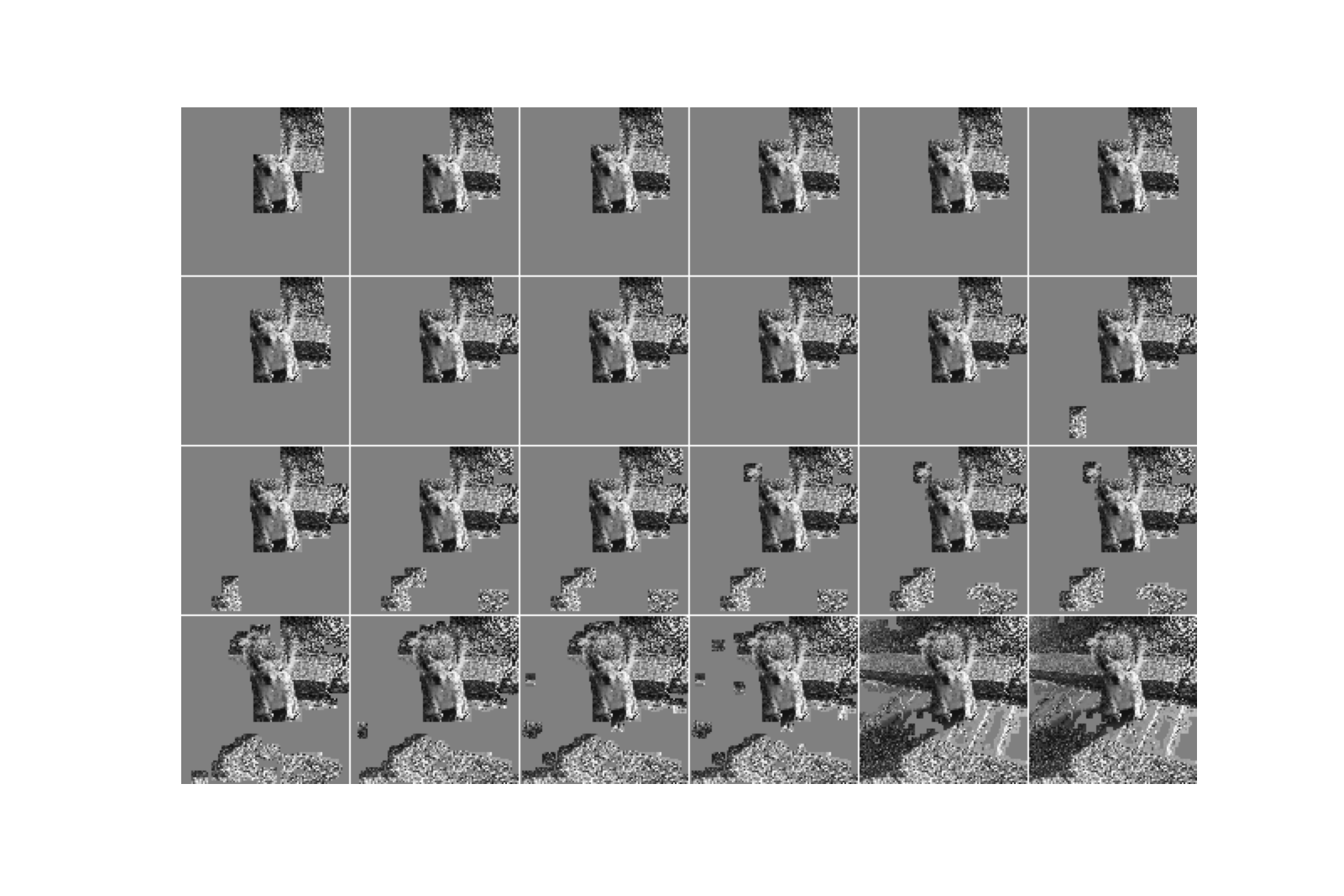}
}
\caption{Results of process using rectangles.}
\label{fig7}
\end{figure}

We at first considered that the detection accuracy will get better by detecting detailed regions if we use a larger class of domains. 
However, although it took much more time for the calculation than the process using squares, it seems that detection accuracy does not change. The reason may be that the increase of rectangles lost the agility on the variation of the regions with respect to the variation of $\varepsilon$,
and detection accuracy does not change. Hence we consider that the accuracy of our method depends heavily on the choice of a family of domains by which we construct the Vietoris-Rips complex.
%

\section{Conclusion and future work}

We constructed the weighted graph $(V_{M,N}, E_{M,N}, w_f)$ from an image $f$, and 
also construct the Vietoris-Rips complex $S_\varepsilon(f)$ from the weighted graph $(V_{M,N}, E_{M,N}, w_f)$ by regarding this weighted graph as a pseudo-metric space. 
We explained that the simplex of the Vietoris-Rips complex $S_\varepsilon(f)$ consists of domains that are closely located and have close numbers of colors. We further explained that the parameter $\varepsilon$ of the Vietoris-Rips complex controls the error of closeness of the numbers of colors. 

By some numerical experiments, we saw that our process can detect salient objects by focusing on squares with minimal size in a maximal simplex of the Vietoris-Rips complex. 
To investigate a suitable class of domains, we conducted two kinds of experiments using square regions and rectangular regions as the classes of domains, respectively. We compared the results, and we found that the detection accuracy of the latter case is not much different from that of the former case, even though the computational cost of the latter case is much higher. 
However, in general, it is expected that the detection accuracy depends on the class of domains. In order to test this hypothesis, we plan to conduct a comparative study in the future with different forms of domains and with a reduced number of domains. 
Investigating various domains includes comparing the method here to the performance of object detection using quadtrees. This is because a quadtree method can be represented as our process with a suitable class of the domains. If we can find an appropriate domain through a comparative study of domains, we will be able to improve the accuracy and reduce the computational cost of our method. In addition, optimizing the hyperparameter $\varepsilon$ in our method is another future work. 

Also, we would like to understand the correspondence between images and the Vietoris-Rips complex constructed from images. 
Let us explain this more formally, in the mathematical language of \textit{category theory},  to make our vision clear. 
The study of image analysis deals with images and operations between them, for example, noise removal, reformation of images, and so on. 
This means that we are working in a category of images. From this viewpoint, we gave in this paper correspondence between objects in the categories of images, simplicial complexes, and further abelian groups by taking homology. 
What we would like to do in the future is to analyze these correspondences as functors from the category of images to category of simplicial complexes or to category of abelian groups. 
That is, we study not only simplicial complexes or abelian groups obtained from images but also corresponding operations among them. For example, we can consider operations between simplicial complexes or abelian groups that correspond to noise removal or reshape. Since the categories of such mathematical objects have been deeply studied so far, we may obtain new and strong tools or observations for image analysis, as we did in this paper. Further, it is very attractive to investigate correspondences between  geometric features of simplicial complexes and feature values of images. We expect that the study of our functor-like construction based on colors and subregions will become one of the new frameworks for understanding images.
\ifCLASSOPTIONcompsoc
  \section*{acknowledgments}
\else
  \section*{Acknowledgment}
\fi

The second author was supported by JSPS KAKENHI Grant Number JP21J12812. 
The third author was supported by RIKEN Center for Advanced Intelligence Project (AIP).
The forth author thanks to KERNEL (DEEPCORE Inc.) for giving him a hospitality and a good environment for researching.
\ifCLASSOPTIONcaptionsoff
  \newpage
\fi



%


\bibliography{ref}
\bibliographystyle{plain}  

\end{document}